\documentclass[11pt]{article}

\usepackage[a4paper,margin=1in]{geometry}
\usepackage[utf8]{inputenc}
\usepackage[T1]{fontenc}
\usepackage{lmodern}
\usepackage{microtype}
\usepackage{amsmath,amssymb,amsfonts,amsthm,mathtools}
\usepackage{graphicx}
\setkeys{Gin}{draft=false}
\usepackage{booktabs}
\usepackage{caption}
\usepackage{subcaption}
\usepackage{float}
\usepackage{placeins}
\usepackage{algorithmicx}
\usepackage{algpseudocode}
\usepackage[most]{tcolorbox}
\usepackage{natbib}
\usepackage{url}
\usepackage{xcolor}
\usepackage{enumitem}
\usepackage[hidelinks]{hyperref}

\newtheorem{assumption}{Assumption}
\newtheorem{lemma}{Lemma}
\newtheorem{proposition}{Proposition}
\newtheorem{theorem}{Theorem}
\newtheorem{corollary}{Corollary}

\newcommand{\R}{\mathbb{R}}
\newcommand{\ip}[2]{\left\langle #1,#2\right\rangle}
\newcommand{\pos}[1]{\left[#1\right]_+}
\newcommand{\normg}[1]{\left\|#1\right\|}

\title{Learning from the Descent Direction: Adaptive Gradient Descent under One-Sided H\"older Regularity}

\author{%
Arzu Ahmadova\\
\small Weierstrass Institute for Applied Analysis and Stochastics\\
\small Berlin, Germany
\and
Ismail Huseynov\\
\small Physikalisch-Technische Bundesanstalt (PTB), Germany\\
\small Technical University of Berlin, Germany
}

\date{}

\begin{document}

\maketitle

\begin{abstract}
We study adaptive gradient descent for continuously differentiable, possibly
nonconvex objectives under one-sided H\"older regularity. Classical
H\"older-gradient and Lipschitz-gradient assumptions control the full norm of
the gradient variation, which can lead to conservative step sizes when much of
that variation is orthogonal to, or favorable along, the descent direction.
Motivated by this observation, we formulate an adaptive gradient method based on a one-sided H\"older curvature estimate, which controls only the directional term that appears in the descent inequality. This directional estimate is weaker than full H\"older gradient continuity and is more directly aligned with the realized update direction. The proposed method chooses a scalar step size from the estimated positive directional curvature and combines it with a simple sufficient-decrease safeguard. For continuously differentiable nonconvex objectives on a convex region containing the accepted update segments, we prove an explicit
best-iterate stationarity guarantee with a rate determined by the H\"older
exponent. Thus, unlike pre-specified diminishing step-size schemes, the method
adapts to the local descent geometry rather than relying only on the iteration
counter. We evaluate the method on two controlled full-batch benchmarks designed to separate positive directional curvature from full gradient variation. In a binary classification benchmark, the proposed one-sided H\"older method achieves
the lowest final cross-entropy, objective value, and gradient norm, together with the largest classification margin among the compared scalar gradient
methods. In a nonconvex H\"older regression benchmark, it obtains the lowest final objective gap and final gradient norm. These results support one-sided H\"older curvature as an effective adaptive step-size signal when full-gradient variation is inflated by directions that do not obstruct descent.
\end{abstract}

% ============================================================
\section{Introduction}
\label{sec:introduction}
% ============================================================

Gradient descent and its variants remain basic tools for large-scale
optimization in machine learning. Their standard nonconvex analysis typically
starts from global Lipschitz continuity of the gradient: there exists
\(L>0\) such that
\begin{equation}
    \label{eq:intro-lsmooth}
    \normg{\nabla F(x)-\nabla F(y)}
    \leq
    L\normg{x-y},
    \qquad
    x,y\in\R^d.
\end{equation}
For differentiable objectives bounded from below, this assumption yields the
classical best-iterate stationarity guarantee
\begin{equation}
     \min_{0\leq k<K}\normg{\nabla F(x_k)}
    =
    O(K^{-1/2})
\end{equation}
for gradient descent with a step size of order \(1/L\)
\citep{ghadimi2013stochastic,carmon2020lower}. However,
\eqref{eq:intro-lsmooth} can be conservative because it controls the full
gradient variation vector. For a descent step, the relevant quantity is not
the entire gradient difference, but the component of this difference that enters
the descent inequality along the realized update direction.

A common relaxation is H\"older continuity of the gradient. Given
\(\alpha\in(0,1]\), one assumes that there exists \(L_\alpha>0\) such that
\begin{equation}
    \label{eq:intro-holder-full}
    \normg{\nabla F(x)-\nabla F(y)}
    \leq
    L_\alpha \normg{x-y}^{\alpha},
    \qquad
    x,y\in\R^d.
\end{equation}
The case \(\alpha=1\) recovers \eqref{eq:intro-lsmooth}. H\"older-gradient
models appear in universal gradient methods and in worst-case complexity
analyses for optimization with non-Lipschitz gradients
\citep{nesterov2015universal,cartis2017holder,yashtini2016global}.
These models are useful because they cover objectives whose gradients are
continuous but not Lipschitz. Nevertheless, they still impose a full-norm
control on the gradient variation, and therefore may overestimate the curvature
that is actually relevant for descent.

Recent work by \citet{ahmadova2025convergence} studies convergence of gradient
descent systems with pre-specified diminishing step-size conditions under
related one-sided regularity assumptions. Such schemes establish convergence to
stationarity once a decay schedule is fixed, but their performance depends on
the chosen decay exponent. This motivates an adaptive alternative: rather
than prescribing a decay schedule as a function of the iteration counter, the
step size should be chosen from the local one-sided curvature observed along
the descent direction.

This paper develops such an adaptive strategy under a one-sided
H\"older curvature condition. We say that \(F\) satisfies a one-sided
H\"older curvature condition with exponent \(\alpha\in(0,1]\) and constant
\(c_\alpha>0\) on a convex set \(Q\) if
\begin{equation}
    \label{eq:intro-one-sided-holder}
    \ip{\nabla F(x)-\nabla F(y)}{x-y}
    \leq
    c_\alpha \normg{x-y}^{1+\alpha},
    \qquad
    x,y\in Q.
\end{equation}
Here \(c_\alpha\) is an upper directional-curvature constant. By the
Cauchy--Schwarz inequality, \eqref{eq:intro-holder-full} implies
\eqref{eq:intro-one-sided-holder} with \(c_\alpha\leq L_\alpha\). The converse
is false in general. The one-sided condition can hold with a much smaller
constant when the full gradient variation is large but mostly orthogonal to, or
favorable along, the displacement direction.

For \(\alpha=1\), \eqref{eq:intro-one-sided-holder} becomes the one-sided
Lipschitz curvature condition
\[
    \ip{\nabla F(x)-\nabla F(y)}{x-y}
    \leq
    c\normg{x-y}^{2}.
\]
For twice continuously differentiable objectives, this controls the largest
eigenvalue of \(\nabla^2F(x)\), whereas the classical Lipschitz-gradient
condition controls the spectral norm. Hence large negative curvature need not
force a small step. Such one-sided Lipschitz or logarithmic-norm type
conditions are classical in stability analysis and adaptive time-stepping, and
they also appear in modern contraction-based learning models
\citep{soderlind2006adaptive,jafarpour2021robust}.

The algorithmic role of \eqref{eq:intro-one-sided-holder} is visible from the
following descent calculation. Let \(g=\nabla F(x)\) and consider
\(x^+=x-\eta g\). Then
\begin{equation}
    \label{eq:intro-descent}
    F(x-\eta g)
    \leq
    F(x)
    -
    \eta\normg{g}^{2}
    +
    \frac{c_\alpha}{1+\alpha}
    \eta^{1+\alpha}\normg{g}^{1+\alpha}.
\end{equation}
Balancing the descent term and the one-sided H\"older remainder gives the
adaptive scale
\begin{equation}
    \label{eq:intro-adaptive-scale}
    \eta
    \asymp
    c_\alpha^{-1/\alpha}
    \normg{\nabla F(x)}^{(1-\alpha)/\alpha}.
\end{equation}
When \(\alpha=1\), this reduces to the usual constant-step smooth regime. When
\(0<\alpha<1\), the step size decreases automatically near stationarity.
The resulting rule is adaptive because the shrinkage is governed by the
current gradient norm and the estimated directional curvature, rather than by a
fixed iteration-dependent schedule.

\paragraph{Contributions.}
The main contributions of this paper are the following.

\begin{enumerate}[leftmargin=1.5em]
    \item \textbf{One-sided H\"older descent model.}
    We formulate gradient descent under the one-sided H\"older curvature
    condition \eqref{eq:intro-one-sided-holder}. This replaces full-gradient
    variation by the positive directional curvature term that appears directly
    in the descent inequality.

    \item \textbf{Adaptive scalar step-size rule.}
    We derive an adaptive one-sided H\"older step scale from
    \eqref{eq:intro-descent}. The rule uses local directional curvature and the
    current gradient norm, so it does not require a pre-specified diminishing
    schedule.

    \item \textbf{Stationarity theory for the safeguarded method.}
    For continuously differentiable nonconvex objectives on a convex
    region containing the accepted update segments, we prove an explicit
    best-iterate stationarity guarantee with the H\"older rate
    \(O(K^{-\alpha/(1+\alpha)})\). We also state the corresponding
    sufficient-decrease and backtracking consequences needed for an
    implementable method.

    \item \textbf{Comparison with scalar curvature baselines.}
    We compare the proposed one-sided H\"older rule with fixed-step GD,
    diminishing-step GD, global Lipschitz, one-sided Lipschitz, and global
    H\"older scalar step rules. These methods are used as baselines to isolate
    the effect of one-sided H\"older curvature, not as additional contributions.

    \item \textbf{Controlled numerical benchmarks.}
    We evaluate the method on two controlled full-batch benchmarks: a
    binary classification benchmark and a nonconvex H\"older regression
    benchmark. In both settings, the design separates positive directional
    curvature from full gradient variation, and OSH achieves the best selected
    stationarity- and loss-based metrics among the compared scalar solvers.
\end{enumerate}

\paragraph{Organization.}
Section~\ref{sec:related-work} reviews related work on line search,
H\"older smoothness, one-sided regularity, and convergence beyond stationarity.
Section~\ref{sec:regularity-solvers} introduces the regularity models,
curvature estimates, and scalar step proposals. Section~\ref{sec:main-theory}
develops the one-sided H\"older descent analysis, the adaptive stationarity
rate, and the safeguarded convergence guarantees. Section~\ref{sec:experiments}
presents the controlled classification and regression benchmarks. Proofs and
reproducibility details are provided in the appendix.

\section{Related work}
\label{sec:related-work}

\paragraph{Line search and adaptive steps.}
Classical line-search methods choose step sizes through sufficient-decrease
conditions rather than through a fixed global smoothness constant. The Armijo
rule \citep{armijo1966} is the standard example. Our method uses this
sufficient-decrease principle only as a safeguard: the initial step proposal is
derived from a one-sided H\"older curvature estimate along the descent
direction. This distinguishes the method from Barzilai--Borwein steps
\citep{barzilai1988}, which use spectral information from successive iterates,
and from local overstepping control \citep{malitsky2020adaptive}, which adapts
steps using gradient differences. It is also distinct from coordinatewise
adaptive methods such as AdaGrad, Adam, AMSGrad, and AdamW
\citep{duchi2011adagrad,kingma2015adam,reddi2018adam,loshchilov2019adamw},
which adapt steps through accumulated gradient or moment information rather
than through a scalar directional-curvature model.

\paragraph{One-sided H\"older regularity.}
\citet{ahmadova2025convergence} studied gradient descent under a one-sided
H\"older regularity condition and proved convergence to stationarity using
pre-specified diminishing step sizes. In contrast, we derive an adaptive
step-size rule from the one-sided H\"older descent inequality itself. The step
therefore depends on the current gradient norm and the estimated directional
curvature, instead of on a prescribed decay exponent. This leads to an explicit best-iterate stationarity rate for the safeguarded adaptive method. A more
detailed comparison between diminishing and adaptive steps is given in
Appendix~\ref{sec:adaptive-vs-diminishing}.

\paragraph{H\"older and generalized smoothness.}
Universal gradient methods adapt to unknown H\"older smoothness in convex
optimization \citep{nesterov2015universal}, while evaluation-complexity results
under H\"older-continuous gradients are known in nonconvex optimization
\citep{cartis2017holder}. Descent analyses beyond Lipschitz gradient continuity have also been developed in broader first-order settings
\citep{bauschke2017descent}. These works rely on full-norm gradient-variation
control. Recent generalized-smoothness and clipping analyses likewise go beyond
uniform Lipschitz-gradient assumptions
\citep{zhang2020clipping,faw2023beyond,gorbunov2025methods}. Our focus is
different: we replace full-gradient variation by the positive directional
curvature term that is sufficient for descent along the realized update.

\paragraph{Beyond stationarity.}
Stationarity rates do not by themselves imply convergence of iterates or
objective-gap rates. Stronger conclusions require additional geometry, such as the Polyak-{\L}ojasiewicz inequality \citep{polyak1963gradient,karimi2016pl},
quadratic growth, error bounds, or Kurdyka-{\L}ojasiewicz structure
\citep{attouch2013convergence,bolte2014palm}. Accordingly, our main
nonconvex guarantee is stated as a best-iterate stationarity result, with
additional consequences recorded only under stronger geometric assumptions.

%============================================================
\section{Background}
\label{sec:regularity-solvers}
%============================================================

Let \(F:\R^d\to\R\) be continuously differentiable. We consider gradient
descent
\[
    x_{k+1}=x_k-\eta_k g_k,\qquad g_k:=\nabla F(x_k),
\]
on a closed convex set \(Q\subset\R^d\) containing the iterates and the accepted
line-search segments.

\subsection{Regularity models}
\label{subsec:four-regularity-models}

The classical full-gradient models and their one-sided counterparts are
summarized in Table~\ref{tab:four-conditions}. The proposed assumption is the
one-sided H\"older curvature condition, which controls the positive directional
curvature entering the descent inequality rather than the full norm of the
gradient variation.

\begin{table}[H]
\centering
\caption{Regularity models and natural step scales. The one-sided models use
only the directional curvature term needed for descent.}
\label{tab:four-conditions}
\begingroup
\small
\setlength{\tabcolsep}{3.5pt}
\renewcommand{\arraystretch}{1.08}
\begin{tabular}{@{}llll@{}}
\toprule
Model & Condition & Step scale & Rate \\
\midrule
Global Lipschitz &
\(\normg{\nabla F(x)-\nabla F(y)}\le L\normg{x-y}\) &
\(\eta_k\asymp L^{-1}\) &
\(K^{-1/2}\) \\
Global H\"older &
\(\normg{\nabla F(x)-\nabla F(y)}\le L_\alpha\normg{x-y}^{\alpha}\) &
\(\eta_k\asymp L_\alpha^{-1/\alpha}\normg{g_k}^{(1-\alpha)/\alpha}\) &
\(K^{-\alpha/(1+\alpha)}\) \\
One-sided Lipschitz &
\(\ip{\nabla F(x)-\nabla F(y)}{x-y}\le c\normg{x-y}^{2}\) &
\(\eta_k\asymp c^{-1}\) &
\(K^{-1/2}\) \\
One-sided H\"older &
\(\ip{\nabla F(x)-\nabla F(y)}{x-y}\le c_\alpha\normg{x-y}^{1+\alpha}\) &
\(\eta_k\asymp c_\alpha^{-1/\alpha}\normg{g_k}^{(1-\alpha)/\alpha}\) &
\(K^{-\alpha/(1+\alpha)}\) \\
\bottomrule
\end{tabular}
\endgroup
\end{table}

\begin{assumption}[One-sided H\"older curvature]
\label{ass:one-sided-holder}
There exist \(c_\alpha>0\) and \(\alpha\in(0,1]\) such that
\[
    \ip{\nabla F(x)-\nabla F(y)}{x-y}
    \leq c_\alpha\normg{x-y}^{1+\alpha},
    \qquad x,y\in Q.
\]
\end{assumption}

By Cauchy--Schwarz, global H\"older gradient continuity,
\(\normg{\nabla F(x)-\nabla F(y)}\le L_\alpha\normg{x-y}^{\alpha}\), implies
Assumption~\ref{ass:one-sided-holder} with \(c_\alpha\le L_\alpha\). The
converse generally fails because the one-sided condition does not control
gradient variation orthogonal to the displacement direction.

A simple quadratic example illustrates the possible gap between full and
one-sided curvature. If \(F(x)=\frac12x^\top A x\) with \(A=A^\top\), then the
full Lipschitz-gradient constant is \(\normg{A}_2\), whereas the one-sided
Lipschitz curvature constant is \(\lambda_{\max}(A)\). Thus for
\(A=\operatorname{diag}(1,-M)\), \(M\gg1\), one has \(L=M\) but \(c=1\). Full
smoothness is dominated by the large negative-curvature direction, while the
one-sided constant measures only the largest positive directional curvature.

\subsection{Curvature estimates and step proposals}
\label{subsec:curvature-estimates}

At iteration \(k\), choose a probe radius \(r_k>0\) and set
\[
    d_k=-r_k\frac{g_k}{\normg{g_k}+\varepsilon}.
\]
With \(\pos{a}:=\max\{a,0\}\), define the local curvature estimates
\[
\begin{aligned}
\widehat L_{1,k}
&:=\frac{\normg{\nabla F(x_k+d_k)-\nabla F(x_k)}}{\normg{d_k}},
&
\widehat c_{1,k}
&:=\frac{\pos{\ip{\nabla F(x_k+d_k)-\nabla F(x_k)}{d_k}}}{\normg{d_k}^{2}},
\\
\widehat L_{\alpha,k}
&:=\frac{\normg{\nabla F(x_k+d_k)-\nabla F(x_k)}}{\normg{d_k}^{\alpha}},
&
\widehat c_{\alpha,k}
&:=\frac{\pos{\ip{\nabla F(x_k+d_k)-\nabla F(x_k)}{d_k}}}{\normg{d_k}^{1+\alpha}}.
\end{aligned}
\]
The positive part is essential: negative directional curvature helps descent
and should not reduce the step size. For stability, each raw estimate
\(\widehat q_k\) is replaced by
\[
    \widetilde q_k=\max\{q_{\min},\lambda\widetilde q_{k-1},\widehat q_k\},
    \qquad \lambda\in(0,1).
\]

The four scalar step proposals used in the experiments are
\[
\begin{aligned}
\eta_k^{\rm GL}
&=\min\left\{\eta_{\max},\frac{s_L}{\widetilde L_{1,k}}\right\},\qquad
\eta_k^{\rm OSL}=\min\left\{\eta_{\max},\frac{s_c}{\widetilde c_{1,k}}\right\},
\\
\eta_k^{\rm GH}
&=\min\left\{\eta_{\max},
s_H\left(\frac{1+\alpha}{\widetilde L_{\alpha,k}}\right)^{1/\alpha}
\normg{g_k}^{(1-\alpha)/\alpha}\right\},\\
\eta_k^{\rm OSH}
&=\min\left\{\eta_{\max},
s_{OH}\left(\frac{1+\alpha}{\widetilde c_{\alpha,k}}\right)^{1/\alpha}
\normg{g_k}^{(1-\alpha)/\alpha}\right\}.
\end{aligned}
\]
The proposed method is the OSH rule. The other three proposals are included as
scalar regularity-based baselines. Each proposal is accepted directly if it
satisfies
\[
    F(x_k-\eta_k g_k)\le F(x_k)-\beta\eta_k\normg{g_k}^{2},
    \qquad \beta\in(0,1),
\]
and otherwise is reduced geometrically until this sufficient-decrease condition
holds.

\begin{tcolorbox}[
    title={Compact OSH pseudocode},
    colback=white,
    colframe=black,
    arc=1mm,
    boxrule=0.7pt,
    left=1mm,
    right=1mm,
    top=1mm,
    bottom=1mm,
    before skip=4pt,
    after skip=4pt
]
\footnotesize
\begin{algorithmic}[1]
\State \textbf{Input:} \(x_0,\alpha,r_k,\rho,\beta,q_{\min},\eta_{\max},s_{OH},\varepsilon\)
\State Initialize \(\widetilde c_{\alpha,-1}=q_{\min}\)
\For{\(k=0,1,2,\ldots\)}
    \State \(g_k=\nabla F(x_k)\)
    \If{\(\normg{g_k}\le \varepsilon\)} \State \textbf{return} \(x_k\) \EndIf
    \State \(d_k=-r_k g_k/(\normg{g_k}+\varepsilon)\)
    \State \(\widehat c_{\alpha,k}=
    \pos{\ip{\nabla F(x_k+d_k)-\nabla F(x_k)}{d_k}}/\normg{d_k}^{1+\alpha}\)
    \State \(\widetilde c_{\alpha,k}=
    \max\{q_{\min},\lambda\widetilde c_{\alpha,k-1},\widehat c_{\alpha,k}\}\)
    \State \(\eta=\min\{\eta_{\max},
    s_{OH}((1+\alpha)/\widetilde c_{\alpha,k})^{1/\alpha}
    \normg{g_k}^{(1-\alpha)/\alpha}\}\)
    \While{\(F(x_k-\eta g_k)>F(x_k)-\beta\eta\normg{g_k}^{2}\)}
        \State \(\eta=\rho\eta\)
    \EndWhile
    \State \(x_{k+1}=x_k-\eta g_k\)
\EndFor
\end{algorithmic}
\end{tcolorbox}

The one-sided H\"older proposal is most useful when
\[
    \frac{\pos{\ip{\nabla F(x_k+d_k)-\nabla F(x_k)}{d_k}}}
    {\normg{\nabla F(x_k+d_k)-\nabla F(x_k)}\normg{d_k}+\varepsilon}
    \ll 1.
\]
In this regime, full-norm estimators can interpret irrelevant or favorable
gradient variation as dangerous curvature, whereas the one-sided estimator
measures only the positive curvature component along the proposed descent
direction.

% ============================================================
\section{Main theory}
\label{sec:main-theory}
% ============================================================

This section gives the descent estimate, the explicit adaptive stationarity
rate, and the corresponding safeguarded rate for the implementable line-search
method. Proofs are deferred to Appendix~\ref{app:proofs}.

\begin{lemma}[One-sided H\"older descent]
\label{lem:osh-descent}
Assume that \(F\in C^1\) and that Assumption~\ref{ass:one-sided-holder} holds
on a convex set containing the segment
\(\{x-tg:0\leq t\leq \eta\}\), where \(g=\nabla F(x)\). Then
\begin{equation}
    \label{eq:osh-descent-basic}
    F(x-\eta g)-F(x)
    \leq
    -\eta\normg{g}^{2}
    +
    \frac{c_\alpha}{1+\alpha}
    \eta^{1+\alpha}\normg{g}^{1+\alpha}.
\end{equation}
\end{lemma}

For a given sufficient-decrease parameter \(\beta\in(0,1)\), define
\begin{equation}
\label{eq:def-a-beta}
    a_\beta
    :=
    \left(
        \frac{(1-\beta)(1+\alpha)}{c_\alpha}
    \right)^{1/\alpha}.
\end{equation}

\begin{theorem}[Explicit adaptive stationarity rate]
\label{thm:explicit-adaptive-rate}
Assume \(F_\star:=\inf_{x\in Q}F(x)>-\infty\) and that
Assumption~\ref{ass:one-sided-holder} holds on \(Q\). Let \(a_\beta\) be defined
by \eqref{eq:def-a-beta}. Suppose
\[
    x_{k+1}=x_k-\eta_k g_k,
    \qquad
    \eta_k=a_\beta\normg{g_k}^{(1-\alpha)/\alpha},
\]
and that all segments \(\{x_k-tg_k:0\leq t\leq \eta_k\}\) remain in \(Q\).
Then, for every \(K\ge1\),
\begin{equation}
    \label{eq:explicit-adaptive-rate}
    \min_{0\leq k<K}\normg{\nabla F(x_k)}
    \leq
    \left(
        \frac{F(x_0)-F_\star}
        {\beta a_\beta K}
    \right)^{\alpha/(1+\alpha)}.
\end{equation}
Consequently,
\[
    \min_{0\leq k<K}\normg{\nabla F(x_k)}
    =
    O\left(K^{-\alpha/(1+\alpha)}\right).
\]
\end{theorem}

\begin{proposition}[Backtracking from a fixed upper step]
\label{prop:backtracking-rate}
Assume \(F_\star>-\infty\) and that Assumption~\ref{ass:one-sided-holder}
holds on all accepted line-search segments. Let the sufficient-decrease
backtracking start from \(\bar\eta>0\) and shrink by \(\rho\in(0,1)\) until
\[
    F(x_k-\eta g_k)
    \leq
    F(x_k)-\beta \eta\normg{g_k}^{2}.
\]
Then the line search terminates at every nonstationary iterate, and the
accepted step satisfies
\[
    \eta_k
    \geq
    \rho
    \min\left\{
        \bar\eta,
        a_\beta\normg{g_k}^{(1-\alpha)/\alpha}
    \right\}.
\]
Moreover, for every \(K\ge1\),
\[
    \min_{0\leq k<K}\normg{\nabla F(x_k)}
    \leq
    \max\left\{
    \left(
        \frac{F(x_0)-F_\star}
        {\beta\rho \bar\eta K}
    \right)^{1/2},
    \left(
        \frac{F(x_0)-F_\star}
        {\beta\rho a_\beta K}
    \right)^{\alpha/(1+\alpha)}
    \right\}.
\]
\end{proposition}

\begin{proposition}[Estimator-to-rate condition for OSH]
\label{prop:estimator-to-rate}
Let the one-sided H\"older proposal be
\[
\bar\eta_k
=
s
\left(
\frac{1+\alpha}{\widetilde c_{\alpha,k}}
\right)^{1/\alpha}
\normg{g_k}^{(1-\alpha)/\alpha},
\qquad s>0,
\]
and suppose that
\[
    q_{\min}\leq \widetilde c_{\alpha,k}\leq C_{\max}
    \qquad\text{for all }k.
\]
If sufficient-decrease backtracking starts from \(\bar\eta_k\) and shrinks by
\(\rho\in(0,1)\), then the accepted step satisfies
\[
\eta_k
\geq
\rho
\min\left\{
s\left(\frac{1+\alpha}{C_{\max}}\right)^{1/\alpha},
a_\beta
\right\}
\normg{g_k}^{(1-\alpha)/\alpha}.
\]
In particular, the accepted OSH steps satisfy the lower step-size condition
used in Theorem~\ref{thm:osh-rate}.
\end{proposition}

\begin{theorem}[Safeguarded adaptive one-sided H\"older rate]
\label{thm:osh-rate}
Assume \(F_\star>-\infty\). Suppose the accepted iterates satisfy the
sufficient-decrease inequality
\[
    F(x_{k+1})
    \leq
    F(x_k)-\beta\eta_k\normg{g_k}^{2},
    \qquad \beta\in(0,1),
\]
and that there exists \(a>0\) such that
\begin{equation}
    \label{eq:adaptive-step-lower}
    \eta_k
    \geq
    a\normg{g_k}^{(1-\alpha)/\alpha}
    \qquad\text{for all }k.
\end{equation}
Then, for every \(K\geq1\),
\[
    \min_{0\leq k<K}\normg{g_k}
    \leq
    \left(
        \frac{F(x_0)-F_\star}
        {\beta a K}
    \right)^{\alpha/(1+\alpha)}.
\]
\end{theorem}

The explicit rule gives the clean H\"older stationarity rate when the
one-sided curvature scale is known. The practical OSH method replaces this
unknown scale by a local one-sided curvature estimate and uses
sufficient-decrease backtracking to prevent unstable steps. Proposition
\ref{prop:estimator-to-rate} shows that, whenever the estimated curvature
remains uniformly bounded above along the trajectory, the accepted OSH steps
retain the same gradient-dependent lower scale required by
Theorem~\ref{thm:osh-rate}.

\begin{theorem}[Global convergence of sufficient-decrease backtracking]
\label{thm:armijo-global}
Assume that \(F\in C^1(\mathbb R^d,\mathbb R)\) is bounded from below and
coercive, that Assumption~\ref{ass:one-sided-holder} holds on
\(\mathbb R^d\), and that every critical point of \(F\) is isolated. Let
\(\bar\eta>0\), \(\rho\in(0,1)\), and \(\beta\in(0,1)\). At each nonstationary
iterate, choose the largest
\[
    \eta_k\in\{\bar\eta,\rho\bar\eta,\rho^2\bar\eta,\ldots\}
\]
such that
\[
    F(x_k-\eta_k\nabla F(x_k))
    \le
    F(x_k)-\beta\eta_k\normg{\nabla F(x_k)}^2.
\]
Then the whole sequence \((x_k)\) converges to a critical point of \(F\). If
\(F\) has a unique critical point \(x_\star\), then \(x_k\to x_\star\).
\end{theorem}

The stationarity guarantees above quantify convergence of the gradient norm.
Objective-gap or distance-to-solution rates require additional geometry. One
standard consequence is the following.

% ============================================================
\section{Numerical experiments}
\label{sec:experiments}
% ============================================================

We evaluate OSH on two controlled full-batch benchmarks designed to separate
positive directional curvature from full gradient variation: a controlled binary
classification problem and a controlled nonconvex H\"older regression problem.
We compare fixed-step GD, diminishing-step GD, global Lipschitz curvature (GL),
one-sided Lipschitz curvature (OSL), global H\"older curvature (GH), and OSH.
All curvature-based methods use the same sufficient-decrease backtracking
safeguard. AdamW is omitted from the main comparison because the paper studies
scalar curvature-adaptive gradient descent, whereas AdamW is a coordinatewise
moment-based optimizer.

All experiments are deterministic full-batch runs. Randomness enters only
through the random seeds used for initialization and, where applicable, the
controlled problem instance. Results are reported as mean \(\pm\) standard
deviation over \(20\) seeds. The optimizer type, number of iterations, objective parameters, and reported metrics are specified below.

For classification, we report final cross-entropy, objective value, gradient
norm, and classification margin. Since all methods reach perfect accuracy on
this controlled task, accuracy is not informative and is not used as a primary
metric. For regression, we report final objective gap, final gradient norm,
median accepted step size, and a curvature diagnostic \(\Gamma_k\). The
stationarity-oriented metrics, especially the final gradient norm, directly
match the convergence theory.

% ------------------------------------------------------------
\subsection{Controlled binary classification}
\label{subsec:controlled-classification}
% ------------------------------------------------------------

The classification objective uses \(x=(u,v_1,\ldots,v_d)\), margin
\[
m(u,v)=b+\frac{s}{d}\sum_{j=1}^d v_j-q u^2,
\qquad
\ell_{\rm cls}(u,v)=\log(1+\exp(-m(u,v))),
\]
and
\[
F_{\rm cls}(u,v)=
\ell_{\rm cls}(u,v)
+\frac{\lambda |u|^{1+\alpha}}{1+\alpha}
+\frac{\gamma}{4}\sum_{j=1}^d (v_j^2-1)^2 .
\]
Thus the classification loss depends on the same variables that generate the
H\"older and nonconvex curvature effects. We use \(\alpha=0.7\), \(d=20\),
\(\lambda=0.5\), \(\gamma=30\), and \(30\) iterations. Table~\ref{tab:controlled-classification-results}
shows that OSH is best on all reported metrics.
\begin{table}[H]
\centering
\caption{Controlled binary classification over \(20\) seeds. Lower is better
except for margin.}
\label{tab:controlled-classification-results}
\scriptsize
\setlength{\tabcolsep}{2.6pt}
\renewcommand{\arraystretch}{0.88}
\resizebox{\linewidth}{!}{%
\begin{tabular}{lcccc}
\toprule
Method & Cross-ent. \(\downarrow\) & Objective \(\downarrow\) &
Grad. norm \(\downarrow\) & Margin \(\uparrow\) \\
\midrule
Fixed GD
& \(5.07{\times}10^{-2} \pm 8.18{\times}10^{-4}\)
& \(9.99{\times}10^{-2} \pm 3.65{\times}10^{-3}\)
& \(3.11{\times}10^{-1} \pm 8.89{\times}10^{-3}\)
& \(2.956 \pm 0.017\) \\
Diminishing GD
& \(5.62{\times}10^{-2} \pm 1.29{\times}10^{-3}\)
& \(1.23{\times}10^{-1} \pm 5.10{\times}10^{-3}\)
& \(3.63{\times}10^{-1} \pm 1.15{\times}10^{-2}\)
& \(2.850 \pm 0.023\) \\
GL
& \(4.24{\times}10^{-1} \pm 5.26{\times}10^{-2}\)
& \(6.42{\times}10^{1} \pm 7.86\)
& \(5.02{\times}10^{1} \pm 6.84{\times}10^{-1}\)
& \(0.644 \pm 0.151\) \\
OSL
& \(4.29{\times}10^{-2} \pm 9.52{\times}10^{-4}\)
& \(8.73{\times}10^{-1} \pm 6.14{\times}10^{-1}\)
& \(8.72 \pm 2.28\)
& \(3.129 \pm 0.022\) \\
GH
& \(4.60{\times}10^{-2} \pm 1.58{\times}10^{-3}\)
& \(7.70{\times}10^{-2} \pm 8.16{\times}10^{-3}\)
& \(2.45{\times}10^{-1} \pm 2.52{\times}10^{-2}\)
& \(3.058 \pm 0.035\) \\
OSH
& \(\mathbf{4.07{\times}10^{-2} \pm 3.02{\times}10^{-3}}\)
& \(\mathbf{4.17{\times}10^{-2} \pm 2.87{\times}10^{-3}}\)
& \(\mathbf{4.91{\times}10^{-2} \pm 1.79{\times}10^{-2}}\)
& \(\mathbf{3.183 \pm 0.066}\) \\
\bottomrule
\end{tabular}%
}
\end{table}

\begin{figure}[!t]
\centering
\begin{subfigure}{0.48\linewidth}
    \centering
    \includegraphics[width=\linewidth]{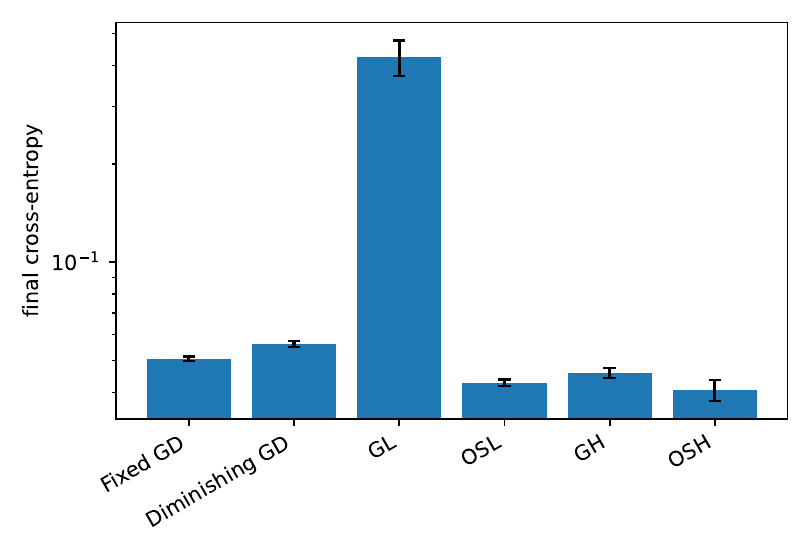}
    \caption{Final cross-entropy.}
\end{subfigure}
\hfill
\begin{subfigure}{0.48\linewidth}
    \centering
    \includegraphics[width=\linewidth]{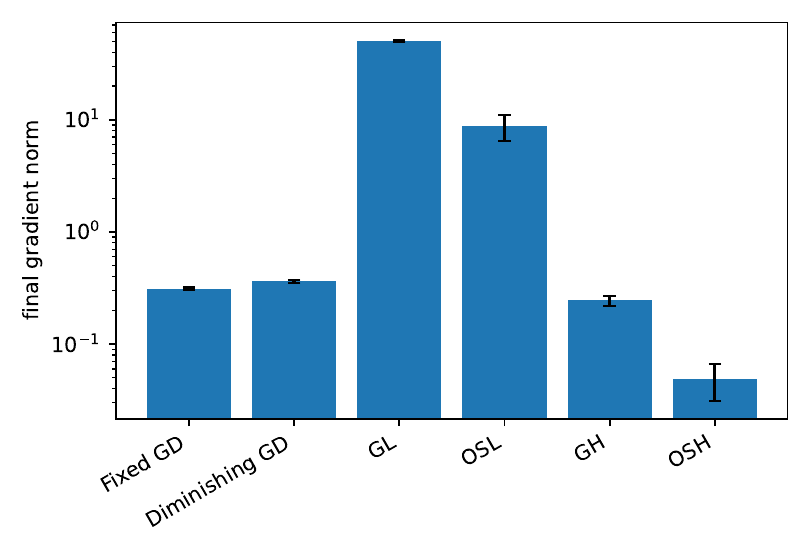}
    \caption{Final gradient norm.}
\end{subfigure}

\begin{subfigure}{0.48\linewidth}
    \centering
    \includegraphics[width=\linewidth]{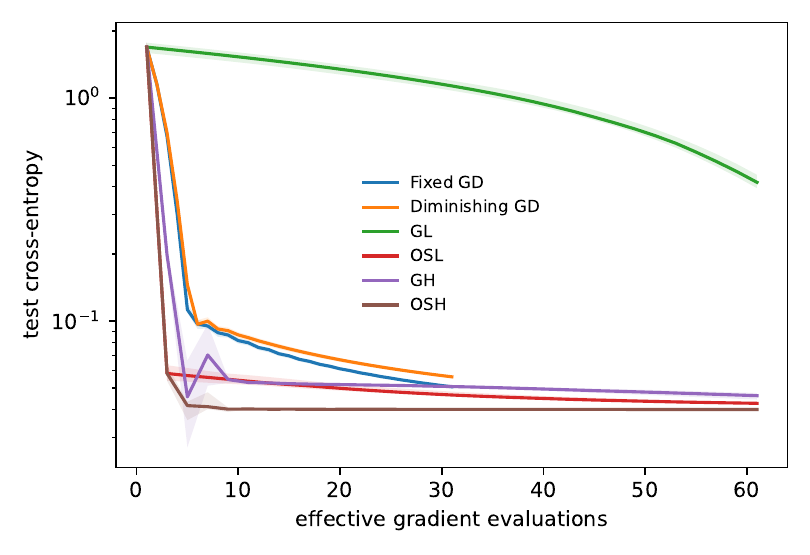}
    \caption{Cross-entropy trajectory.}
\end{subfigure}
\hfill
\begin{subfigure}{0.48\linewidth}
    \centering
    \includegraphics[width=\linewidth]{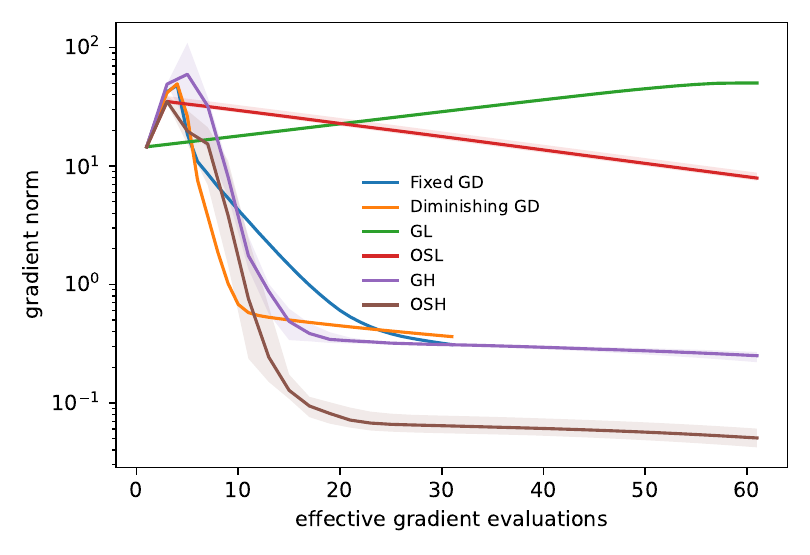}
    \caption{Gradient-norm trajectory.}
\end{subfigure}

\caption{Controlled binary classification. The top row reports final metric
summaries over \(20\) seeds, and the bottom row reports median trajectories with
interquartile bands. OSH achieves the lowest final cross-entropy and gradient
norm and shows the strongest stationarity behavior.}
\label{fig:controlled-classification}
\vspace{-0.5em}
\end{figure}

\FloatBarrier

% ------------------------------------------------------------
\subsection{Controlled nonconvex H\"older regression}
\label{subsec:controlled-regression}
% ------------------------------------------------------------

The regression benchmark is
\[
F_{\rm reg}(u,v)=
\frac{|u|^{1+\alpha}}{1+\alpha}
+\frac{\gamma}{4}\sum_{j=1}^{d}(v_j^2-1)^2,
\qquad
\alpha=\frac12,
\]
with \(F_\star=0\), \(d=20\), \(\gamma=10\), and initialization
\(u_0\approx1\), \(v_{0,j}\approx0.1\). We run \(50\) iterations. This benchmark
directly tests the stationarity mechanism because the H\"older term determines
the gradient-dependent step scaling, while the nonconvex quartic terms create
directions where full-norm curvature can be less informative than one-sided
directional curvature. Table~\ref{tab:controlled-regression-results} shows that
OSH gives the lowest final objective gap and final gradient norm.
\begin{table}[H]
\centering
\caption{Controlled nonconvex H\"older regression over \(20\) seeds. Lower is
better for final gap and gradient norm.}
\label{tab:controlled-regression-results}
\scriptsize
\setlength{\tabcolsep}{2.6pt}
\renewcommand{\arraystretch}{0.88}
\resizebox{\linewidth}{!}{%
\begin{tabular}{lcccc}
\toprule
Method & Final gap \(\downarrow\) & Final grad. norm \(\downarrow\) &
Median step & Median \(\Gamma_k\) \\
\midrule
Fixed GD
& \(1.04{\times}10^{-5} \pm 1.74{\times}10^{-21}\)
& \(2.50{\times}10^{-2} \pm 3.56{\times}10^{-18}\)
& \(5.00{\times}10^{-2} \pm 7.12{\times}10^{-18}\)
& -- \\
Diminishing GD
& \(1.89{\times}10^{-6} \pm 2.17{\times}10^{-22}\)
& \(1.41{\times}10^{-2} \pm 0.00\)
& \(3.96{\times}10^{-2} \pm 7.12{\times}10^{-18}\)
& -- \\
GL
& \(1.46{\times}10^{-7} \pm 6.37{\times}10^{-7}\)
& \(1.06{\times}10^{-3} \pm 3.72{\times}10^{-3}\)
& \(3.39{\times}10^{-2} \pm 1.14{\times}10^{-2}\)
& \(1.009 \pm 0.013\) \\
OSL
& \(5.13{\times}10^{-7} \pm 9.00{\times}10^{-7}\)
& \(4.37{\times}10^{-3} \pm 6.30{\times}10^{-3}\)
& \(3.61{\times}10^{-2} \pm 1.58{\times}10^{-2}\)
& \(1.005 \pm 0.017\) \\
GH
& \(1.59{\times}10^{-9} \pm 1.23{\times}10^{-9}\)
& \(7.45{\times}10^{-4} \pm 2.29{\times}10^{-4}\)
& \(2.70{\times}10^{-3} \pm 1.06{\times}10^{-3}\)
& \(1.041 \pm 0.030\) \\
OSH
& \(\mathbf{2.52{\times}10^{-10} \pm 4.92{\times}10^{-10}}\)
& \(\mathbf{3.06{\times}10^{-4} \pm 3.30{\times}10^{-4}}\)
& \(1.08{\times}10^{-2} \pm 1.83{\times}10^{-2}\)
& \(1.013 \pm 0.031\) \\
\bottomrule
\end{tabular}%
}
\end{table}
Across both benchmarks, OSH is best on the metrics aligned with the theory. In
classification, it gives the lowest final cross-entropy, objective value, and
gradient norm, and the largest margin. In nonconvex H\"older regression, it gives the lowest final objective gap and gradient norm. These results indicate that one-sided H\"older curvature can provide a useful scalar step-size signal when
full-gradient variation is inflated by directions that are irrelevant or
favorable for descent.

\begin{figure}[!t]
\centering
\begin{subfigure}{0.48\linewidth}
    \centering
    \includegraphics[width=\linewidth]{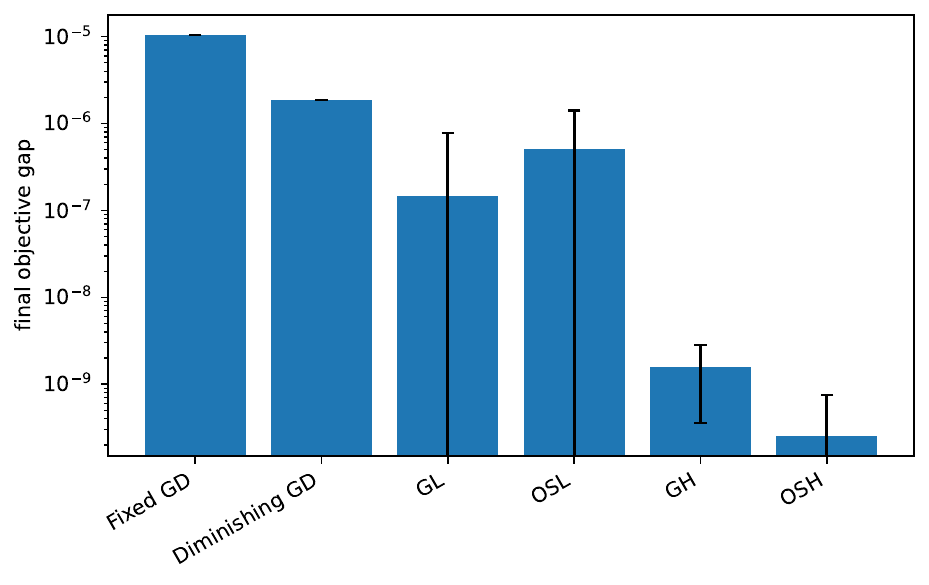}
    \caption{Final objective gap.}
\end{subfigure}
\hfill
\begin{subfigure}{0.45\linewidth}
    \centering
    \includegraphics[width=\linewidth]{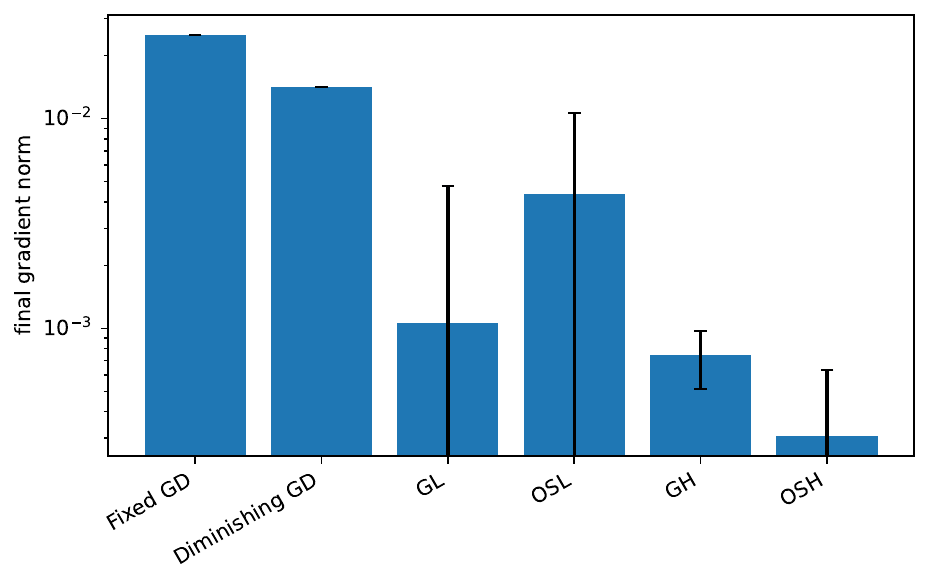}
    \caption{Final gradient norm.}
\end{subfigure}

\begin{subfigure}{0.48\linewidth}
    \centering
    \includegraphics[width=\linewidth]{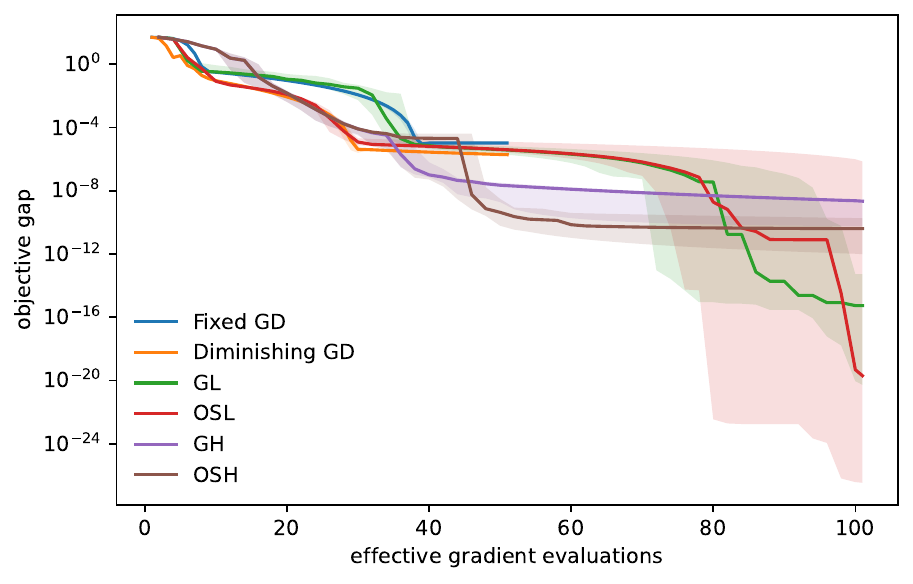}
    \caption{Objective-gap trajectory.}
\end{subfigure}
\hfill
\begin{subfigure}{0.48\linewidth}
    \centering
    \includegraphics[width=\linewidth]{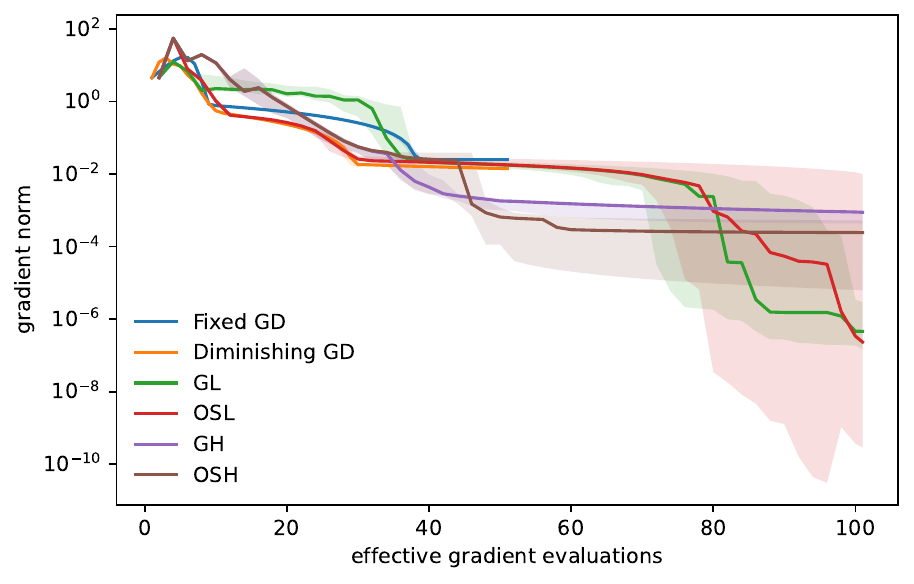}
    \caption{Gradient-norm trajectory.}
\end{subfigure}

\vspace{-0.2em}
\caption{Controlled nonconvex H\"older regression. The top row reports final
metric summaries over \(20\) seeds, and the bottom row reports median
trajectories with interquartile bands. OSH obtains the lowest final objective
gap and final gradient norm.}
\label{fig:controlled-regression}

\end{figure}

\FloatBarrier

% ============================================================
\section{Conclusion}
\label{sec:conclusion}
\paragraph{Discussion.}
We introduced an adaptive scalar gradient method based on one-sided H\"older curvature. The central idea is to estimate only the positive directional curvature that enters the descent inequality, rather than the full norm of the gradient variation. This distinction is important when full-gradient variation is dominated by components that are orthogonal to the descent direction or favorable for descent. Under one-sided H\"older regularity, we proved a best-iterate stationarity guarantee for the adaptive step-size rule, together with safeguarded variants that preserve sufficient decrease in practical
implementations. The numerical experiments support the proposed mechanism. In the controlled binary classification benchmark, OSH achieves the lowest final cross-entropy, objective value, and gradient norm, and the largest final margin. In the controlled nonconvex H\"older regression benchmark, OSH obtains the
lowest final objective gap and final gradient norm. These results indicate that one-sided H\"older curvature can provide a more informative scalar step-size signal than full-norm curvature estimates when the optimization geometry is strongly direction-dependent.
\paragraph{Limitations.}
The experiments are intentionally controlled and full-batch, so they should be
viewed as mechanism-focused evidence rather than a broad empirical comparison
against modern stochastic optimizers. In particular, coordinatewise
moment-based methods such as AdamW are outside the main scope of this paper,
because the theory concerns scalar curvature-adaptive gradient descent. The
proposed method also requires additional gradient evaluations for curvature
probing and may therefore be more expensive per iteration than fixed-step or
diminishing-step gradient descent. The theoretical results assume that the
accepted line-search segments remain in a region where one-sided H\"older
regularity holds, and the practical estimator-to-rate result depends on
boundedness of the adaptive curvature estimate. Future work should study
stochastic variants, lower-cost curvature probes, data-driven choices of the
probe radius, and larger-scale learning problems where one-sided curvature
effects can be separated from minibatch noise and model-selection effects.

{\small
\bibliographystyle{plainnat}
\bibliography{references}
}

\appendix

% ============================================================
\section{Proofs}
\label{app:proofs}
% ============================================================

%\subsection{Proof of Proposition~\ref{prop:quadratic-gap}}

%For a quadratic objective,
%\[   \nabla F(x)-\nabla F(y)=A(x-y).
%\]
%Hence\[\normg{\nabla F(x)-\nabla F(y)}\leq\normg{A}_2\normg{x-y},\] and the best such constant is \(\normg{A}_2\). Also, \[  \ip{\nabla F(x)-\nabla F(y)}{x-y} =(x-y)^\top A(x-y)\leq\lambda_{\max}(A)\normg{x-y}^{2}.\] For \(A=\operatorname{diag}(1,-M)\), \(\normg{A}_2=M\) and \(\lambda_{\max}(A)=1\). Therefore the full Lipschitz constant is \(M\), while the one-sided Lipschitz curvature constant is \(1\).

\subsection{Proof of Lemma~\ref{lem:osh-descent}}

By the fundamental theorem of calculus,
\begin{align}
    F(x-\eta g)-F(x)
    &=
    -\int_0^\eta
    \ip{\nabla F(x-tg)}{g}\,dt
    \\
    &=
    -\eta\normg{g}^{2}
    +
    \int_0^\eta
    \ip{\nabla F(x)-\nabla F(x-tg)}{g}\,dt.
\end{align}
Applying Assumption~\ref{ass:one-sided-holder} with \(y=x-tg\) gives
\[
    \ip{\nabla F(x)-\nabla F(x-tg)}{tg}
    \leq
    c_\alpha t^{1+\alpha}\normg{g}^{1+\alpha}.
\]
Dividing by \(t>0\) and integrating yields
\[
    \int_0^\eta
    \ip{\nabla F(x)-\nabla F(x-tg)}{g}\,dt
    \leq
    c_\alpha\normg{g}^{1+\alpha}
    \int_0^\eta t^\alpha\,dt.
\]
Therefore
\[
    F(x-\eta g)-F(x)
    \leq
    -\eta\normg{g}^{2}
    +
    \frac{c_\alpha}{1+\alpha}
    \eta^{1+\alpha}\normg{g}^{1+\alpha}.
\]

\subsection{Proof of Theorem~\ref{thm:explicit-adaptive-rate}}

Using Lemma~\ref{lem:osh-descent} with
\[
    \eta_k=a_\beta\normg{g_k}^{(1-\alpha)/\alpha}
\]
gives
\[
    F(x_{k+1})-F(x_k)
    \leq
    -\eta_k\normg{g_k}^{2}
    +
    \frac{c_\alpha}{1+\alpha}
    \eta_k^{1+\alpha}\normg{g_k}^{1+\alpha}.
\]
Since
\[
    \frac{c_\alpha}{1+\alpha}a_\beta^\alpha=1-\beta,
\]
we have
\[
    F(x_{k+1})
    \leq
    F(x_k)-\beta a_\beta\normg{g_k}^{(1+\alpha)/\alpha}.
\]
Summing from \(k=0\) to \(K-1\) gives
\[
    \beta a_\beta
    \sum_{k=0}^{K-1}
    \normg{g_k}^{(1+\alpha)/\alpha}
    \leq
    F(x_0)-F_\star.
\]
Taking the minimum and raising to the power \(\alpha/(1+\alpha)\) proves the
claim.
\subsection{Polyak–Łojasiewicz inequality and quadratic growth onsequences}
The global stationarity result shows that the adaptive method drives the
gradient norm to zero under one-sided H\"older regularity. To obtain rates for
the objective gap or for the distance to the solution set, additional local geometric
structure is necessary. We therefore record two standard consequences under the
Polyak--{\L}ojasiewicz inequality and quadratic growth.
\begin{corollary}[PL and quadratic-growth consequences]
\label{cor:pl-qg-rates}
Assume the hypotheses of Theorem~\ref{thm:explicit-adaptive-rate}. If \(F\)
satisfies the Polyak--{\L}ojasiewicz inequality
\[
    \normg{\nabla F(x)}^2
    \geq
    2\mu(F(x)-F_\star),
\]
then, for \(0<\alpha<1\),
\[
    F(x_k)-F_\star
    =
    O\left(k^{-2\alpha/(1-\alpha)}\right).
\]
If \(\alpha=1\), the objective gap decreases linearly:
\[
    F(x_k)-F_\star
    \leq
    (1-2\mu\beta a_\beta)^k(F(x_0)-F_\star),
\]
provided \(0<2\mu\beta a_\beta<1\). If, in addition, \(F\) satisfies quadratic
growth,
\[
    F(x)-F_\star
    \geq
    \frac{m}{2}\operatorname{dist}(x,X_\star)^2,
    \qquad
    X_\star:=\arg\min F,
\]
then, for \(0<\alpha<1\),
\[
    \operatorname{dist}(x_k,X_\star)
    =
    O\left(k^{-\alpha/(1-\alpha)}\right).
\]
\end{corollary}

\begin{proof}

Let \(\Delta_k=F(x_k)-F_\star\). The adaptive descent inequality gives
\[
    \Delta_{k+1}
    \leq
    \Delta_k
    -
    \beta a_\beta\normg{\nabla F(x_k)}^{(1+\alpha)/\alpha}.
\]
Under the PL inequality,
\[
    \normg{\nabla F(x_k)}^{(1+\alpha)/\alpha}
    \geq
    (2\mu)^{(1+\alpha)/(2\alpha)}
    \Delta_k^{(1+\alpha)/(2\alpha)}.
\]
For \(0<\alpha<1\), this yields
\[
    \Delta_{k+1}
    \leq
    \Delta_k-A\Delta_k^\theta,
    \qquad
    \theta=\frac{1+\alpha}{2\alpha}>1,
\]
which implies
\[
    \Delta_k=O(k^{-1/(\theta-1)})
    =
    O(k^{-2\alpha/(1-\alpha)}).
\]
For \(\alpha=1\), \(\theta=1\), so the recursion is linear:
\[
    \Delta_{k+1}
    \leq
    (1-2\mu\beta a_\beta)\Delta_k.
\]
Quadratic growth converts objective-gap convergence to distance convergence.
\end{proof}
\subsection{Proof of Theorem~\ref{thm:armijo-global}}

We write
\[
g_n:=\nabla F(x_n),
\qquad
F_n:=F(x_n).
\]
If \(g_n=0\) for some \(n\), then \(x_n\) is already a critical point. In that
case, we may set
\[
x_m=x_n
\qquad\text{for all }m\ge n,
\]
and the conclusion is immediate. Hence, in the rest of the proof, we assume
that
\[
g_n\neq0
\qquad\text{for all }n\ge0.
\]

\medskip

\noindent
\textbf{Step 1: Descent estimate from one-sided H\"older regularity.}

Fix \(x\in\mathbb{R}^d\) and write
\[
g:=\nabla F(x).
\]
For \(\gamma>0\), define
\[
\phi(t):=F(x-tg),
\qquad t\in[0,\gamma].
\]
Since \(F\in C^1(\mathbb{R}^d,\mathbb{R})\), the fundamental theorem of calculus
gives
\[
F(x-\gamma g)-F(x)
=
-\int_0^\gamma
\langle \nabla F(x-tg),g\rangle\,dt.
\]
Adding and subtracting \(\nabla F(x)\), we obtain
\[
F(x-\gamma g)-F(x)
=
-\gamma\|g\|^2
+
\int_0^\gamma
\langle \nabla F(x)-\nabla F(x-tg),g\rangle\,dt.
\]
Now apply the one-sided H\"older condition with
\[
y=x-tg.
\]
Then
\[
x-y=tg.
\]
Therefore,
\[
\langle \nabla F(x)-\nabla F(x-tg),tg\rangle
\le
c_\alpha\|tg\|^{1+\alpha}.
\]
For \(t>0\), dividing by \(t\) gives
\[
\langle \nabla F(x)-\nabla F(x-tg),g\rangle
\le
c_\alpha t^\alpha \|g\|^{1+\alpha}.
\]
Hence
\[
F(x-\gamma g)-F(x)
\le
-\gamma\|g\|^2
+
c_\alpha\|g\|^{1+\alpha}
\int_0^\gamma t^\alpha\,dt.
\]
Since
\[
\int_0^\gamma t^\alpha\,dt
=
\frac{\gamma^{1+\alpha}}{1+\alpha},
\]
we obtain
\[
F(x-\gamma g)
\le
F(x)
-
\gamma\|g\|^2
+
\frac{c_\alpha}{1+\alpha}\gamma^{1+\alpha}\|g\|^{1+\alpha}.
\tag{1}
\label{eq:basic-armijo-descent}
\]

\medskip

\noindent
\textbf{Step 2: The Armijo line search terminates.}

We show that the Armijo condition holds for all sufficiently small
\(\gamma>0\). By \eqref{eq:basic-armijo-descent}, it is enough to require
\[
-\gamma\|g\|^2
+
\frac{c_\alpha}{1+\alpha}\gamma^{1+\alpha}\|g\|^{1+\alpha}
\le
-\sigma\gamma\|g\|^2.
\]
Equivalently,
\[
\frac{c_\alpha}{1+\alpha}\gamma^{1+\alpha}\|g\|^{1+\alpha}
\le
(1-\sigma)\gamma\|g\|^2.
\]
Since \(g\neq0\) and \(\gamma>0\), this is equivalent to
\[
\frac{c_\alpha}{1+\alpha}\gamma^\alpha \|g\|^{\alpha-1}
\le
1-\sigma.
\]
Define
\[
\eta
:=
\left(
\frac{(1-\sigma)(1+\alpha)}{c_\alpha}
\right)^{1/\alpha}.
\]
Then the above inequality holds whenever
\[
\gamma
\le
\eta\|g\|^{\frac{1-\alpha}{\alpha}}.
\tag{2}
\label{eq:safe-armijo-threshold}
\]
Since the trial step sizes
\[
\bar\gamma,\rho\bar\gamma,\rho^2\bar\gamma,\ldots
\]
converge to \(0\), eventually one of them satisfies
\eqref{eq:safe-armijo-threshold}. Therefore the Armijo line search terminates
at every nonstationary iterate.

\medskip

\noindent
\textbf{Step 3: Lower bound for the accepted step size.}

For each \(n\), define
\[
\delta_n
:=
\eta\|g_n\|^{\frac{1-\alpha}{\alpha}}.
\]
By Step 2, every trial step \(\gamma\le\delta_n\) satisfies the Armijo
condition. Since \(\gamma_n\) is chosen as the largest accepted point in the
geometric grid, we claim that
\[
\gamma_n
\ge
\rho\min\{\bar\gamma,\delta_n\}.
\tag{3}
\label{eq:armijo-step-lower-bound}
\]

Indeed, if \(\bar\gamma\le\delta_n\), then the first trial step
\(\bar\gamma\) is accepted, and therefore
\[
\gamma_n=\bar\gamma
\ge
\rho\bar\gamma
=
\rho\min\{\bar\gamma,\delta_n\}.
\]
If \(\bar\gamma>\delta_n\), choose the smallest integer \(m\ge1\) such that
\[
\rho^m\bar\gamma\le\delta_n.
\]
Then
\[
\delta_n<\rho^{m-1}\bar\gamma.
\]
The trial step \(\rho^m\bar\gamma\) is accepted, so
\[
\gamma_n\ge \rho^m\bar\gamma.
\]
Also,
\[
\rho\delta_n
<
\rho^m\bar\gamma.
\]
Thus
\[
\gamma_n
\ge
\rho^m\bar\gamma
>
\rho\delta_n
=
\rho\min\{\bar\gamma,\delta_n\}.
\]
Hence \eqref{eq:armijo-step-lower-bound} holds in all cases.

\medskip

\noindent
\textbf{Step 4: Monotonicity of the objective.}

By the Armijo condition,
\[
F_{n+1}
=
F(x_n-\gamma_n g_n)
\le
F_n-\sigma\gamma_n\|g_n\|^2.
\]
Therefore
\[
F_{n+1}\le F_n.
\]
Hence \((F_n)\) is monotonically nonincreasing. Since \(F\) is bounded from
below, \((F_n)\) converges to some finite limit.

\medskip

\noindent
\textbf{Step 5: Boundedness of the iterates.}

Since \(F_n\le F_0\) for every \(n\), all iterates belong to the sublevel set
\[
S_0
:=
\{x\in\mathbb{R}^d:F(x)\le F(x_0)\}.
\]
Because \(F\) is coercive, \(S_0\) is bounded. Since \(F\) is continuous,
\(S_0\) is closed. Therefore \(S_0\) is compact.

Thus \((x_n)\) is bounded.

\medskip

\noindent
\textbf{Step 6: A summability estimate.}

Combining Armijo decrease with the lower bound
\eqref{eq:armijo-step-lower-bound}, we obtain
\[
F_{n+1}
\le
F_n-\sigma\gamma_n\|g_n\|^2
\le
F_n
-
\sigma\rho
\min\{\bar\gamma,\delta_n\}\|g_n\|^2.
\]
Using the definition of \(\delta_n\), we have
\[
\min\{\bar\gamma,\delta_n\}\|g_n\|^2
=
\min\left\{
\bar\gamma\|g_n\|^2,
\eta\|g_n\|^{\frac{1-\alpha}{\alpha}}\|g_n\|^2
\right\}.
\]
Therefore,
\[
\min\{\bar\gamma,\delta_n\}\|g_n\|^2
=
\min\left\{
\bar\gamma\|g_n\|^2,
\eta\|g_n\|^{\frac{1+\alpha}{\alpha}}
\right\}.
\]
Hence
\[
F_{n+1}
\le
F_n
-
\sigma\rho
\min\left\{
\bar\gamma\|g_n\|^2,
\eta\|g_n\|^{\frac{1+\alpha}{\alpha}}
\right\}.
\tag{4}
\label{eq:armijo-decrease-min}
\]

Summing \eqref{eq:armijo-decrease-min} from \(n=0\) to \(N\), we get
\[
\sigma\rho
\sum_{n=0}^{N}
\min\left\{
\bar\gamma\|g_n\|^2,
\eta\|g_n\|^{\frac{1+\alpha}{\alpha}}
\right\}
\le
F_0-F_{N+1}.
\]
Since \(F\) is bounded from below,
\[
F_{N+1}\ge \inf_{\mathbb{R}^d}F.
\]
Therefore,
\[
\sigma\rho
\sum_{n=0}^{N}
\min\left\{
\bar\gamma\|g_n\|^2,
\eta\|g_n\|^{\frac{1+\alpha}{\alpha}}
\right\}
\le
F_0-\inf_{\mathbb{R}^d}F.
\]
Letting \(N\to\infty\), we conclude that
\[
\sum_{n=0}^{\infty}
\min\left\{
\bar\gamma\|g_n\|^2,
\eta\|g_n\|^{\frac{1+\alpha}{\alpha}}
\right\}
<\infty.
\tag{5}
\label{eq:armijo-summability}
\]

\medskip

\noindent
\textbf{Step 7: The gradients vanish.}

We prove that
\[
\|g_n\|\to0.
\]
Suppose, for contradiction, that \(\|g_n\|\not\to0\). Then there exist
\(\varepsilon>0\) and a subsequence \((g_{n_j})\) such that
\[
\|g_{n_j}\|\ge\varepsilon
\qquad\text{for all }j.
\]
For every \(j\),
\[
\min\left\{
\bar\gamma\|g_{n_j}\|^2,
\eta\|g_{n_j}\|^{\frac{1+\alpha}{\alpha}}
\right\}
\ge
\min\left\{
\bar\gamma\varepsilon^2,
\eta\varepsilon^{\frac{1+\alpha}{\alpha}}
\right\}.
\]
The right-hand side is a strictly positive constant. This contradicts the
summability in \eqref{eq:armijo-summability}. Therefore
\[
\|g_n\|\to0.
\]
Equivalently,
\[
\|\nabla F(x_n)\|\to0.
\tag{6}
\label{eq:gradient-vanishes-armijo}
\]

\medskip

\noindent
\textbf{Step 8: Successive increments vanish.}

Since every accepted step size belongs to the trial grid,
\[
0<\gamma_n\le\bar\gamma.
\]
Therefore,
\[
\|x_{n+1}-x_n\|
=
\gamma_n\|g_n\|
\le
\bar\gamma\|g_n\|.
\]
Using \eqref{eq:gradient-vanishes-armijo}, we obtain
\[
\|x_{n+1}-x_n\|\to0.
\tag{7}
\label{eq:increments-vanish-armijo}
\]

\medskip

\noindent
\textbf{Step 9: The cluster set is nonempty, compact, and connected.}

Define the cluster set
\[
\omega(x_n)
:=
\left\{
\bar x\in\mathbb{R}^d:
\exists\, n_j\to\infty
\text{ such that }
x_{n_j}\to \bar x
\right\}.
\]
Since \((x_n)\) is bounded, \(\omega(x_n)\) is nonempty and compact.

Moreover, by \eqref{eq:increments-vanish-armijo},
\[
\|x_{n+1}-x_n\|\to0.
\]
A standard connectedness lemma for cluster sets states that a bounded sequence
with vanishing successive increments has a connected cluster set. Therefore
\[
\omega(x_n)
\]
is connected.

\medskip

\noindent
\textbf{Step 10: Every cluster point is critical.}

Let \(\bar x\in\omega(x_n)\). Then there exists a subsequence
\((x_{n_j})\) such that
\[
x_{n_j}\to \bar x.
\]
By \eqref{eq:gradient-vanishes-armijo},
\[
\nabla F(x_{n_j})\to0.
\]
Since \(F\in C^1(\mathbb{R}^d,\mathbb{R})\), the gradient \(\nabla F\) is
continuous. Hence
\[
\nabla F(\bar x)
=
\lim_{j\to\infty}\nabla F(x_{n_j})
=
0.
\]
Thus every cluster point is critical:
\[
\omega(x_n)
\subseteq
\operatorname{Crit}(F),
\]
where
\[
\operatorname{Crit}(F)
:=
\{x\in\mathbb{R}^d:\nabla F(x)=0\}.
\]

\medskip

\noindent
\textbf{Step 11: Isolation of critical points gives a single cluster point.}

Since all iterates lie in \(S_0\), we have
\[
\omega(x_n)
\subseteq
\operatorname{Crit}(F)\cap S_0.
\]
The set \(S_0\) is compact. Because \(\nabla F\) is continuous,
\(\operatorname{Crit}(F)\) is closed. Hence
\[
\operatorname{Crit}(F)\cap S_0
\]
is compact.

By assumption, every critical point of \(F\) is isolated. Therefore
\(\operatorname{Crit}(F)\cap S_0\) is a compact set consisting only of isolated
points, and hence it is finite.

Since \(\omega(x_n)\) is connected and contained in a finite set, it must be a
singleton. Therefore there exists a critical point \(x_\infty\) such that
\[
\omega(x_n)=\{x_\infty\}.
\]

\medskip

\noindent
\textbf{Step 12: The whole sequence converges.}

We now show that
\[
x_n\to x_\infty.
\]
Suppose not. Then there exist \(\varepsilon>0\) and a subsequence
\((x_{n_j})\) such that
\[
\|x_{n_j}-x_\infty\|\ge\varepsilon
\qquad\text{for all }j.
\]
Since \((x_n)\) is bounded, the subsequence \((x_{n_j})\) has a further
convergent subsequence, say
\[
x_{n_{j_\ell}}\to \widetilde x.
\]
Then \(\widetilde x\in\omega(x_n)\). Since
\[
\omega(x_n)=\{x_\infty\},
\]
we get
\[
\widetilde x=x_\infty.
\]
However, passing to the limit in
\[
\|x_{n_{j_\ell}}-x_\infty\|\ge\varepsilon
\]
gives
\[
\|\widetilde x-x_\infty\|\ge\varepsilon,
\]
which contradicts \(\widetilde x=x_\infty\). Therefore
\[
x_n\to x_\infty.
\]
Since \(x_\infty\in\operatorname{Crit}(F)\), we have
\[
\nabla F(x_\infty)=0.
\]

If \(F\) has a unique critical point \(x_\star\), then necessarily
\[
x_\infty=x_\star.
\]
Thus
\[
x_n\to x_\star.
\]
This completes the proof.

\subsection{Proof of Proposition~\ref{prop:backtracking-rate}}

By Lemma~\ref{lem:osh-descent}, Armijo decrease is guaranteed whenever
\[
    \frac{c_\alpha}{1+\alpha}
    \eta^\alpha\normg{g_k}^{\alpha-1}
    \leq
    1-\beta.
\]
Equivalently,
\[
    \eta
    \leq
    a_\beta\normg{g_k}^{(1-\alpha)/\alpha}.
\]
Therefore backtracking terminates. Since the line search uses a geometric grid,
the accepted step satisfies
\[
    \eta_k
    \geq
    \rho
    \min\{
        \bar\eta,
        a_\beta\normg{g_k}^{(1-\alpha)/\alpha}
    \}.
\]
Combining this lower bound with Armijo decrease gives
\[
    F(x_{k+1})
    \leq
    F(x_k)
    -
    \beta\rho
    \min\{
        \bar\eta\normg{g_k}^2,
        a_\beta\normg{g_k}^{(1+\alpha)/\alpha}
    \}.
\]
Summing over \(k=0,\ldots,K-1\) and using \(F(x_K)\ge F_\star\) gives
\[
    \sum_{k=0}^{K-1}
    \min\{
        \bar\eta\normg{g_k}^2,
        a_\beta\normg{g_k}^{(1+\alpha)/\alpha}
    \}
    \leq
    \frac{F(x_0)-F_\star}{\beta\rho}.
\]
Let \(m_K=\min_{0\le k<K}\normg{g_k}\). Then
\[
    K
    \min\{
        \bar\eta m_K^2,
        a_\beta m_K^{(1+\alpha)/\alpha}
    \}
    \leq
    \frac{F(x_0)-F_\star}{\beta\rho}.
\]
The stated bound follows.

\subsection{Proof of Proposition~\ref{prop:estimator-to-rate}}

The proposal satisfies
\[
\bar\eta_k
\geq
s
\left(
\frac{1+\alpha}{C_{\max}}
\right)^{1/\alpha}
\normg{g_k}^{(1-\alpha)/\alpha}.
\]
Backtracking accepts a step no smaller than \(\rho\) times the minimum of the
initial proposal and the largest guaranteed Armijo-safe H\"older step
\(a_\beta\normg{g_k}^{(1-\alpha)/\alpha}\). This gives the claimed lower bound.

\subsection{Proof of Theorem~\ref{thm:osh-rate}}

By the assumed sufficient-decrease inequality,
\[
    F(x_{k+1})
    \leq
    F(x_k)-\beta\eta_k\normg{g_k}^{2}.
\]
Using the lower step-size bound in \eqref{eq:adaptive-step-lower},
\[
    \eta_k
    \geq
    a\normg{g_k}^{(1-\alpha)/\alpha},
\]
we obtain
\[
    F(x_{k+1})
    \leq
    F(x_k)
    -
    \beta a\normg{g_k}^{(1+\alpha)/\alpha},
\]
because
\[
    2+\frac{1-\alpha}{\alpha}
    =
    \frac{1+\alpha}{\alpha}.
\]
Summing from \(k=0\) to \(K-1\) yields
\[
    \beta a
    \sum_{k=0}^{K-1}
    \normg{g_k}^{(1+\alpha)/\alpha}
    \leq
    F(x_0)-F(x_K)
    \leq
    F(x_0)-F_\star.
\]
Let
\[
    m_K:=\min_{0\leq k<K}\normg{g_k}.
\]
Then
\[
    K m_K^{(1+\alpha)/\alpha}
    \leq
    \sum_{k=0}^{K-1}
    \normg{g_k}^{(1+\alpha)/\alpha}.
\]
Consequently,
\[
    \beta a K m_K^{(1+\alpha)/\alpha}
    \leq
    F(x_0)-F_\star.
\]
Raising both sides to the power \(\alpha/(1+\alpha)\) gives
\[
    \min_{0\leq k<K}\normg{g_k}
    \leq
    \left(
        \frac{F(x_0)-F_\star}{\beta a K}
    \right)^{\alpha/(1+\alpha)}.
\]
This proves the result.

\section{Why Adaptive Steps Instead of Diminishing Steps?}
\label{sec:adaptive-vs-diminishing}
% ============================================================
Assume that \(F\in C^1(\mathbb{R}^d,\mathbb{R})\) is bounded from below, the
iterates remain in a bounded convex set, and the one-sided H\"older condition
holds on this set. Let
\[
\eta_k=\eta_0(k+1)^{-p}.
\]
If \(\alpha=1\), assume \(0<p\le1\). If \(0<\alpha<1\), the one-sided H\"older diminishing-step theory requires 
\[
\frac{1-\alpha}{1+\alpha}<p\le1.
\]
Then
\[
\|\nabla F(x_k)\|\to0.
\]
Moreover, for \(0<p<1\), the corresponding best-iterate stationarity rate has the
form
\[
\min_{0\le k<K}\|\nabla F(x_k)\|
=
O\left(K^{-(1-p)/2}\right),
\]
and for \(p=1\),
\[
\min_{0\le k<K}\|\nabla F(x_k)\|
=
O\left((\log K)^{-1/2}\right).
\]
The best exponent is approached only as
\[
    p\downarrow \frac{1-\alpha}{1+\alpha},
\]
which formally yields the limiting exponent \(\alpha/(1+\alpha)\).

The adaptive method instead balances
\[
    -\eta_k\normg{g_k}^{2}
\]
against the one-sided H\"older remainder
\[
    \eta_k^{1+\alpha}\normg{g_k}^{1+\alpha}.
\]
This gives
\[
    \eta_k
    \asymp
    \normg{g_k}^{(1-\alpha)/\alpha}.
\]
Thus the step size shrinks because the gradient norm indicates proximity to
stationarity, not merely because the iteration counter is large. This is both a
theoretical and practical advantage: the method uses the local descent geometry
rather than a pre-selected decay exponent.

% ============================================================
\section{Experimental Details}
\label{app:experimental-details}
% ============================================================

The experiments are full-batch deterministic optimization runs on controlled
objectives. No stochastic mini-batches, train/test splits, or validation sets
are used, because both benchmarks are objective-based rather than
dataset-based. Thus, the reported quantities are objective-side and
stationarity-side metrics computed directly along the optimization trajectory.
The reported variability comes from \(20\) random seeds, which determine the
initialization and the controlled problem instance when applicable.

For the controlled binary classification benchmark, the objective is
\(F_{\rm cls}\) from Section~\ref{subsec:controlled-classification}. We use
\(\alpha=0.7\), \(d=20\), \(\lambda=0.5\), \(\gamma=30\), and run each method
for \(30\) iterations. Since all methods reach perfect classification accuracy
on this controlled task, accuracy is not used as a primary metric. We instead
report final binary cross-entropy, objective value, gradient norm, and
classification margin.

For the controlled nonconvex H\"older regression benchmark, the objective is
\(F_{\rm reg}\) from Section~\ref{subsec:controlled-regression}. We use
\(\alpha=1/2\), \(d=20\), \(\gamma=10\), initialization \(u_0\approx1\),
\(v_{0,j}\approx0.1\), and run each method for \(50\) iterations. The global
minimum value is \(F_\star=0\), so the final objective gap is computed as
\(F(x_K)-F_\star\).

The compared optimizers are fixed-step GD, diminishing-step GD, GL, OSL, GH,
and OSH. All methods use full gradients. The curvature-based methods use the
same sufficient-decrease backtracking criterion and differ only in the scalar
curvature estimate used to propose the initial step size. Hyperparameters are
chosen from the grids implemented in the supplementary code, using the same
selection protocol for all methods within each benchmark.

\paragraph{Compute resources.}
All experiments are deterministic full-batch experiments on low-dimensional
objectives and were run on a single CPU worker. No GPU, TPU, or distributed
compute was used. The controlled binary classification benchmark uses
\(20\) random seeds and \(30\) iterations per method, while the controlled
nonconvex H\"older regression benchmark uses \(20\) random seeds and \(50\)
iterations per method. On a standard laptop or workstation CPU, each full
benchmark sweep runs within a few minutes and requires less than 1~GB of RAM.
The plotting and table-generation scripts run in under one minute.

\end{document}